\documentclass{article}

    \PassOptionsToPackage{numbers, compress}{natbib}

\usepackage[preprint]{neurips_2020_arxiv}

\usepackage[utf8]{inputenc} 
\usepackage[T1]{fontenc}    
\usepackage{hyperref}       
\usepackage{url}            
\usepackage{booktabs}       
\usepackage{amsfonts}       
\usepackage{nicefrac}       
\usepackage{microtype}      

\usepackage{graphicx}
\usepackage{amsfonts}
\usepackage{amsmath}
\usepackage{amssymb}
\usepackage{amsthm}
\usepackage{algorithmic}
\usepackage[ruled,vlined,noend]{algorithm2e}
\usepackage{wrapfig}
\usepackage{bm}

\newcommand{\KL}[2]{\textit{KL}(#1\mid\mid#2)}
\newcommand{\PM}[1]{\mathbb{P}_{#1}} 
\newcommand{\PJ}[2]{\mathbb{P}_{#1#2}}  
\newcommand{\PI}[2]{\mathbb{P}_{#1}\otimes\mathbb{P}_{#2}}  
\newcommand\PP{\mathbb{P}}

\newcommand{\Tau}{\mathcal{T}}
\newcommand{\FGen}{T}
\newcommand\RR{\mathbb{R}}
\newcommand\EE{\mathbb{E}}
\newcommand{\ent}{\mathcal{H}}
\newcommand{\SN}{\FGen_{\phi}}
\newtheoremstyle{questionstyle}
  {\topsep}   
  {0}         
  {\itshape}  
  {0pt}       
  {\bfseries} 
  {.}         
  {5pt plus 1pt minus 1pt} 
  {}          
\theoremstyle{questionstyle}\newtheorem{question}{Question}
\newtheorem{theorem}{Theorem}
\newtheorem{lemma}[theorem]{Lemma}

\newcommand\blfootnote[1]{%
  \begingroup
  \renewcommand\thefootnote{}\footnote{#1}%
  \addtocounter{footnote}{-1}%
  \endgroup
}

\title{Mutual Information-based State-Control for Intrinsically Motivated Reinforcement Learning}

\author{Rui Zhao$^{1}$
\And Yang Gao$^{2}$  
\And Pieter Abbeel$^{2}$ 
\And Volker Tresp$^{1}$ 
\And Wei Xu$^{3}$}

\begin{document}

\maketitle

\begin{abstract}
In reinforcement learning, an agent learns to reach a set of goals by means of an external reward signal. In the natural world, intelligent organisms learn from internal drives, bypassing the need for external signals, which is beneficial for a wide range of tasks. Motivated by this observation, we propose to formulate an intrinsic objective as the mutual information between the goal states and the controllable states. This objective encourages the agent to take control of its environment. Subsequently, we derive a surrogate objective of the proposed reward function, which can be optimized efficiently. Lastly, we evaluate the developed framework in different robotic manipulation and navigation tasks and demonstrate the efficacy of our approach. A video showing experimental results is available at \url{https://youtu.be/CT4CKMWBYz0}.\blfootnote{
$^{1}$Siemens AG \& Ludwig Maximilian University of Munich. 
$^{2}$University of California, Berkeley.
$^{3}$Horizon Robotics.
Correspondence to: Rui Zhao {\tt\small $\lbrace$zhaorui.in.germany@gmail.com$\rbrace$}.}
\end{abstract}


\section{Introduction}
In psychology~\cite{sansone2000intrinsic}, behavior is considered intrinsically motivated when it originates from an internal drive. 
An intrinsic motivation is essential to develop behaviors required for accomplishing a broad range of tasks rather than solving a specific problem guided by an external reward.

Intrinsically motivated reinforcement learning (RL)~\cite{chentanez2005intrinsically} equips an agent with various internal drives via intrinsic rewards, such as curiosity~\cite{schmidhuber1991possibility,pathak2017curiosity,burda2018large}, diversity~\cite{gregor2016variational,haarnoja2018soft,eysenbach2018diversity}, and empowerment~\cite{klyubin2005empowerment,salge2014empowerment,mohamed2015variational}, which allow the agent to develop meaningful behaviors for solving a wide range of tasks. 
Mutual information (MI) is a core statistical quantity that has many applications in intrinsically motivated RL.
\citet{still2012information} calculate the curiosity bonus based on the MI between the past and the future states within a time series.
\citet{mohamed2015variational} developed a scalable approach to calculate a common internal drive known as empowerment, which is defined as the channel capacity between the states and the actions.
\citet{eysenbach2018diversity} use the MI between skills and states as an intrinsic reward to help the agent to discover a diverse set of skills.
In multi-goal RL \cite{schaul2015universal,andrychowicz2017hindsight,plappert2018multi}, \citet{warde2018unsupervised} propose to utilize the MI between the high-dimensional observation and the goals as the reward signal to help the agent to learn goal-conditioned policies with visual inputs.
To discover skills and learn the dynamics of these skills for model-based RL, \citet{sharma2020dynamics} recently designed an approach based on MI between the next state and the current skill, conditioned on the current state.

In this paper, we investigate the idea that agent's ``preparedness'' to control the states to reach any potential goal would be an effective intrinsic motivation for RL agents.
We formulate the ``preparedness'' of control as the MI between the goal states and agent's controllable states.
This internal drive extends agent's controllability from controllable states to goal states and subsequently prepares the agent to reach any goal.
It makes learning possible in the absence of hand-engineered reward functions or manually-specified goals. 
Furthermore, learning to ``master'' the environment potentially helps the agent to learn in sparse reward settings. 
We propose a new unsupervised RL method called Mutual Information-based State-Control (MISC).
During the learning process of the agent, an MI estimator is trained to evaluate the MI between the goal states and agent's controllable states. 
Concurrently, the agent is rewarded for maximizing the MI estimation.

\begin{figure*}
  \centering
  \includegraphics[width=4.8 in]{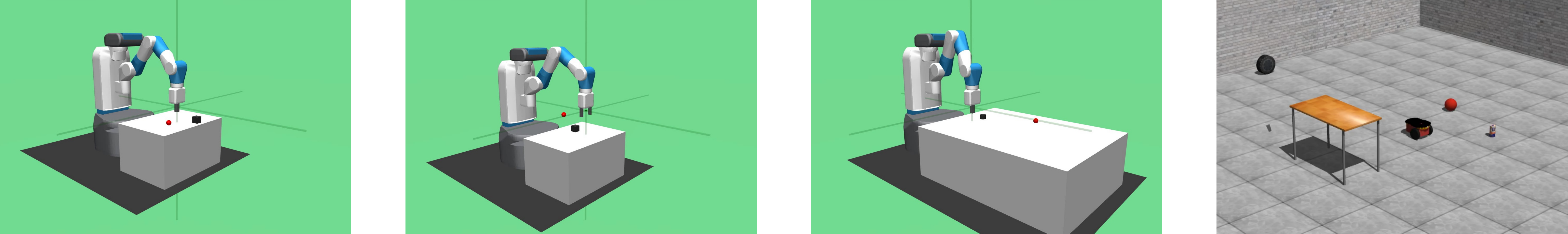}
  \caption{Fetch robot arm manipulation tasks in OpenAI Gym and a navigation task based on the Gazebo simulator: 
\texttt{FetchPush}, \texttt{FetchPickAndPlace}, \texttt{FetchSlide}, \texttt{SocialBot-PlayGround}.}
  \label{fig:fetch3env1nav}
\end{figure*}

This paper contains the following five contributions.
First, we introduce MISC for intrinsically motivated RL.
Secondly, we derive a scalable MI surrogate objective for optimization.
Thirdly, we evaluate the developed framework for the robotic tasks of manipulation and navigation and demonstrate the control behavior that agents learned purely via the intrinsic reward.
Fourthly, incorporating the intrinsic reward with the task reward, we compare our approach with state-of-the-art methods.
Last but not least, we observe that the learned MI estimator from one task can be transferred to a different task and still accelerate learning.


\section{Preliminaries}

We consider multi-goal RL tasks, like the robotic simulation scenarios provided by OpenAI Gym~\cite{plappert2018multi}, where four tasks are used for evaluation, including push, slide, pick \& place with the robot arm, and a newly designed navigation task with a mobile robot in Gazebo~\cite{koenig2004design}, as shown in Figure~\ref{fig:fetch3env1nav}.
Accordingly, we define the following terminologies for these scenarios. 

\subsection{Goal States, Controllable States, and Reinforcement Learning Settings} 
The goal in the manipulation task is to move the target object to a desired position.
For the navigation task, the goal for the robot is to navigate to the target ball.
These goals are sampled from the environment at the beginning of each episode.
Note that in this paper we consider that the goals can be represented by states~\cite{andrychowicz2017hindsight}, which leads us to the concept of goal states $s^g$.

In this paper, the \textbf{goal state} $s^g$ refers to the state variable that the agent is interested in. For example, it can be the position of the target object in a manipulation task.  
A related but different concept is the \textbf{environment goal} $g^e$, which is a desired value of the goal state in the episode. For example, it is a particular goal position of the target object in the manipulation task. 
The \textbf{controllable state} $s^c$ is the state that can be directly influenced by the agent~\cite{borsa2019observational}, such as the state of the robot and its end-effector.
The goal states and the controllable states are mutually exclusive.
The state split is under the designer's control.
In this paper, we use upper letters, such as $S$, to denote random variables and the corresponding lower case letter, such as $s$, to represent the values of random variables.
We consider an agent interacting with an environment. We assume the environment is fully observable, including a set of state $\mathcal{S}$, a set of action $\mathcal{A}$, a distribution of initial states $p(s_0)$, transition probabilities $p(s_{t+1} \mid s_t, a_t)$, a reward function $r$: $\mathcal{S} \times \mathcal{A} \rightarrow \mathbb{R}$. 


\section{Method}
\label{sec:misc}
\begin{figure*}
    \centering
    \begin{minipage}{0.45\linewidth}
        \includegraphics[width=\linewidth]{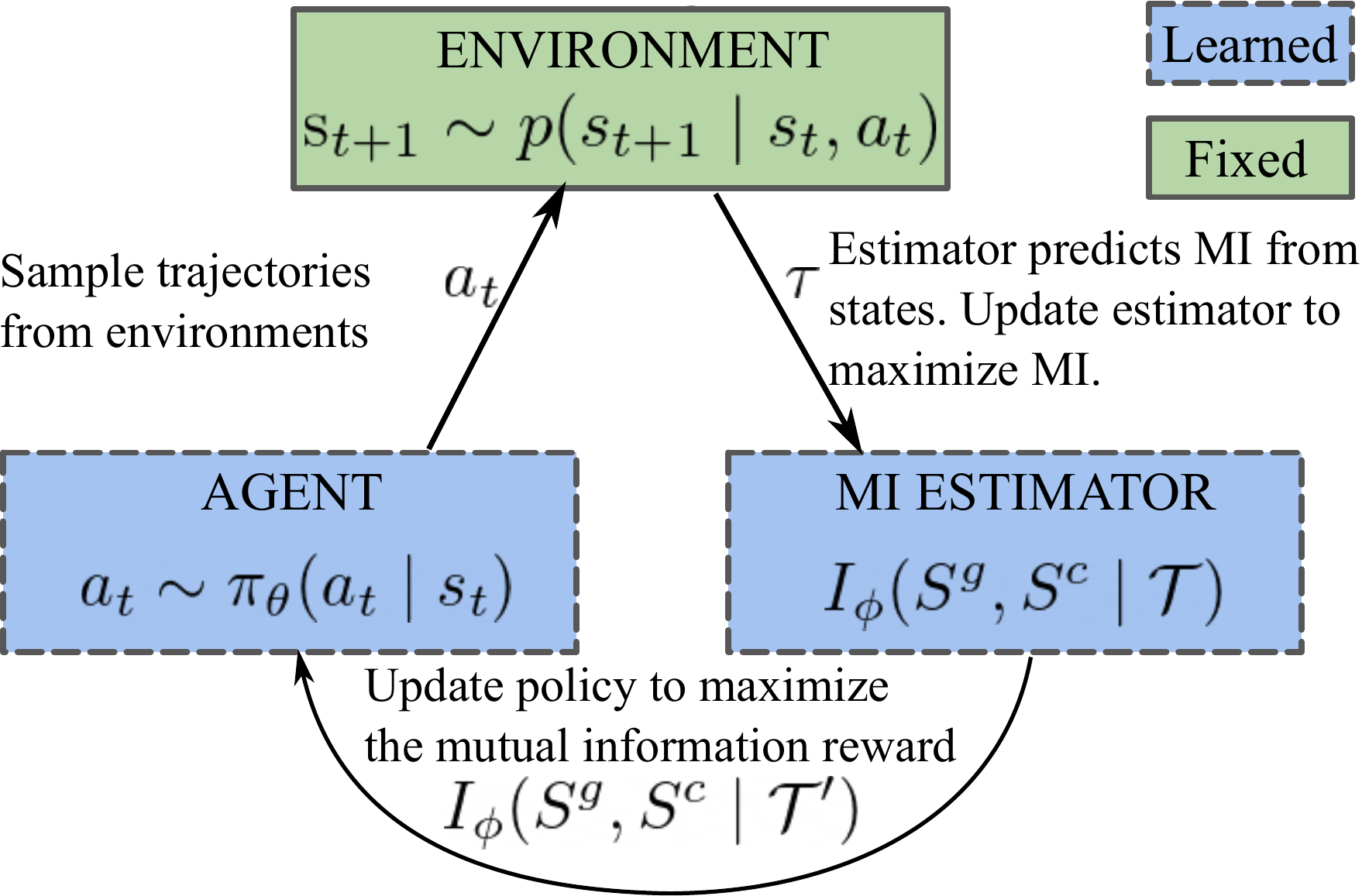}
    \end{minipage} \hfill
    \begin{minipage}{0.54\linewidth}
    
    \begin{algorithm}[H]
    \DontPrintSemicolon
    \SetAlgoLined
    \While{not converged}{
        Sample an initial state $s_0 \sim p(s_0)$.\\
        \For{$t \leftarrow 1$ \KwTo $steps\_per\_episode$}{
                Sample action $a_t \sim \pi_{\theta}(a_t \mid s_t)$.\\
                Step environment $s_{t+1} \sim p(s_{t+1} \mid s_t, a_t)$.\\
                Sample transitions $\Tau'$ from the buffer.\\
                Set intrinsic reward $r=I_{\phi}(S^g; S^c \mid \Tau')$.\\
                Update policy ($\theta$) via DDPG or SAC.\\
                Update the MI estimator ($\phi$) with SGD.
              }
    } 
    \caption{MISC}\label{algo:misc}
    \end{algorithm}
    \end{minipage}
    \caption{\textbf{MISC Algorithm}:
    We update the estimator to better predict the MI, and update the agent to control goal states to have higher MI with the controllable states. \label{fig:misc}}
\end{figure*}

We focus on agents learning to control goal states purely by using its observations and actions without supervision.
Motivated by the idea that an agent should be ``prepared'' to control the goal state with its own directly controllable state, we formulate the problem of learning without external supervision as one of learning a policy $\pi_{\theta}(a_t \mid s_t)$ with parameters $\theta$ to maximize intrinsic MI rewards, $r = I(S^g; S^c)$. 
In this section, we formally describe our method, MI-based state control.

\subsection{Mutual Information Reward Function}

Our framework simultaneously learns a policy and an intrinsic reward function by maximizing the MI between the goal states and the controllable states.
Mathematically, the MI between the goal state random variable $S^g$ and the controllable state random variable $S^c$ is represented as follows:
\begin{align}
            \label{eq:mi-1}
I(S^g; S^c) &= H(S^g) - H(S^g \mid S^c) \\
            \label{eq:mi-2}
            &= \KL{\PJ{S^g}{S^c}}{\PI{S^g}{S^c}} \\
            \label{eq:mi-donsker}
            &= \sup_{\FGen : \Omega \to \RR} \EE_{\PJ{S^g}{S^c}}[\FGen] - \log(\EE_{\PI{S^g}{S^c}}[e^{\FGen}]) \\
            \label{eq:mi-lb}
            &\geq\sup_{\phi\in\Phi} \EE_{\PJ{S^g}{S^c}}[\FGen_\phi] - \log(\EE_{\PI{S^g}{S^c}}[e^{\FGen_\phi}])
            = I_{\Phi}(S^g; S^c), 
\end{align}
where $\PJ{S^g}{S^c}$ is the joint probability distribution; $\PI{S^g}{S^c}$ is the product of the marginal distributions $\PM{S^g}$ and $\PM{S^c}$; $\textit{KL}$ denotes the Kullback-Leibler (KL) divergence.
Equation~(\ref{eq:mi-1}) tells us that the agent should maximize the entropy of goal states $H(S^g)$, and concurrently, should minimize the conditional entropy of goal states given the controllable states $H(S^g \mid S^c)$. When the conditional entropy $H(S^g \mid S^c)$ is small, it becomes easy to predict the goal states based on the controllable states. 
Equation~(\ref{eq:mi-2}) gives us the MI in the KL divergence form.

MI is notoriously difficult to compute in real-world settings~\cite{hjelm2018learning}.
Motivated by MINE~\cite{belghazi2018mine}, we use a lower bound to approximate the MI quantity $I(S^g; S^c)$. 
First, we rewrite Equation~(\ref{eq:mi-2}), the KL formulation of the MI objective, using the Donsker-Varadhan representation, to Equation~(\ref{eq:mi-donsker})~\cite{donsker1975asymptotic}.
The input space $\Omega$ is a compact domain of $\RR^d$, i.e., $\Omega \subset \RR^d$, and the supremum is taken over all functions $\FGen$ such that the two expectations are finite.
Secondly, we lower bound the MI in the Donsker-Varadhan representation with the compression lemma in the PAC-Bayes literature and then derive Equation~(\ref{eq:mi-lb})~\cite{banerjee2006bayesian,belghazi2018mine}.
The expectations in Equation~(\ref{eq:mi-lb}) are estimated by using empirical samples from $\PJ{S^g}{S^c}$ and $\PI{S^g}{S^c}$.
We can also sample the marginal distributions by shuffling the samples from the joint distribution along the axis~\cite{belghazi2018mine}.
The derived MI reward function, $r=I_{\Phi}(S^g; S^c)$, can be trained by gradient ascent. The statistics model $\FGen_{\phi}$ is parameterized by a deep neural network with parameters $\phi \in \Phi$, which is capable of estimating the MI with arbitrary accuracy.

\subsection{Efficient Learning State-Control}

At the beginning of each episode, the agent takes actions $a_{t}$ following a partially random policy, such as $\epsilon$-greedy, to explore the environment and collects trajectories into a replay buffer.
The trajectory $\tau$ contains a series of states, $\tau=\{s_1, s_2, \ldots, s_{t^*}\}$, where $t^{*}$ is the time horizon of the trajectory.
Its random variable is denoted as $\Tau$.
Each state $s_t$ consists of goal states $s_t^g$ and controllable states $s_t^c$. 

For training the MI estimator network, we first randomly sample the trajectory $\tau$ from the replay buffer.
Then, the states $s_{t}^{c}$ used for calculating the product of marginal distributions are sampled by shuffling the states $s_{t}^{c}$ from the joint distribution along the temporal axis $t$ within the trajectory, see Equation~(\ref{eq:mi-eq-surrogate}) Left-Hand Side (LHS).
Note that we calculate the MI by using the samples from the same trajectory. 
If the agent does not alter the goal states during the episode, then the MI between the goal states and the controllable states remains zero.

We use back-propagation to optimize the parameter $(\phi)$ to maximize the MI lower bound, see Equation~(\ref{eq:mi-eq-surrogate}) LHS. 
However, for evaluating the MI, this lower bound, Equation~(\ref{eq:mi-eq-surrogate}) LHS, is time-consuming to calculate because it needs to process on all the samples from the whole trajectory.
To improve its scalability and efficiency, we derive a surrogate objective, Equation~(\ref{eq:mi-eq-surrogate}) Right-Hand Side (RHS), which is computed much more efficiently.
Each time, to calculate the MI reward for the transition $r=I_{\phi}(S^g; S^c \mid \Tau')$, the new objective only needs to calculate over a small fraction of the complete trajectory, $\tau'$. 
The trajectory fraction, $\tau'$, is defined as adjacent state pairs, $\tau'=\{s_{t}, s_{t+1}\}$, and $\Tau'$ represents its corresponding random variable.

\begin{lemma}
The mutual information quantity $I_{\phi}(S^g; S^c \mid \Tau)$ increases when we maximize the surrogate objective $\EE_{\PM{\Tau'}} [ I_{\phi}(S^g; S^c \mid \Tau')]$, mathematically,  
\begin{align}
\label{eq:mi-eq-surrogate}
I_{\phi}(S^g; S^c \mid \Tau) \ltimes \EE_{\PM{\Tau'}} [ I_{\phi}(S^g; S^c \mid \Tau')],
\end{align}
where $S^g$, $S^c$, and $\Tau$ denote goal states, controllable states, and trajectories, respectively. The trajectory fractions are defined as the adjacent state pairs, namely $\Tau'=\{S_{t}, S_{t+1}\}$.
The symbol $\ltimes$ denotes a monotonically increasing relationship between two variables and $\phi$ represents the parameter of the statistics model in MINE. $\textit{Proof.}$ See Appendix~\ref{app:proof-surrogate}. {\hfill $\square$}
\end{lemma}

The derived MI surrogate objective, Equation~(\ref{eq:mi-eq-surrogate}) RHS, brings us two important benefits. 
First, it enables us to estimate the MI reward for each transition with much less computational time because we only use the trajectory fraction, instead of the trajectory. 
This approximately reduces the complexity from $\mathcal{O}(t^*)$ to $\mathcal{O}(1)$ with respect to the trajectory length $t^{*}$.
Secondly, this way of estimating MI also enables us to assign rewards more accurately at the transition level because now we use only the relevant state pair to calculate the transition reward.

Formally, we define the transition MI reward as the MI estimation of each trajectory fraction, namely
\begin{align}
\begin{split}
r_{\phi}(a_t,s_t)
:= I_{\phi}(S^g; S^c | \Tau') 
= 0.5 {\textstyle \sum}_{i=t}^{t+1} \SN({s}^{g}_i, {s}^{c}_i) - \log(0.5 {\textstyle \sum}_{i=t}^{t+1} e^{\SN({s}^{g}_i,
      \bar{{s}}^{c}_i)}),
\end{split}  
\end{align}
where $({s}^{g}_i, {s}^{c}_i) \sim \PJ{S^g}{S^c\mid \Tau'} $,  $\bar{{s}}^{c}_i \sim \PP_{S^c\mid \Tau'} $, and $\tau'=\{s_{t}, s_{t+1}\}$. 
In case that the estimated MI value is particularly small, we scale the reward with a hyper-parameter $\alpha$ and clip the reward between 0 and 1.

\textbf{Implementation:}
We combine MISC with both deep deterministic policy gradient (DDPG)~\cite{lillicrap2015continuous} and soft actor-critic (SAC)~\cite{haarnoja2018soft} to learn a policy $\pi_{\theta}(a \mid s)$ that aims to control the goal states.
In comparison to DDPG and SAC, the DDPG method improves the policy in a more ``greedy'' fashion, while the SAC approach is more conservative, in the sense that SAC incorporates an entropy regularizer $\ent(A \mid S)$ that maximizes the policy's entropy over actions.
We summarize the complete training algorithm in Algorithm~\ref{algo:misc} and in Figure~\ref{fig:misc}.

\textbf{MISC Variants with Task Rewards:}
We propose three ways of using MISC to accelerate learning in addition to the task reward.
The first method is using the MISC pretrained policy as the parameter initialization and fine-tuning the agent with rewards. 
We denote this variant as ``MISC-f'', where ``-f'' stands for fine-tuning. 
The second variant is to use the MI intrinsic reward to help the agent to explore high MI states. 
We name this method as ``MISC-r'', where ``-r'' stands for reward. 
The third approach is to use the MI quantity from MISC to prioritize trajectories for replay. 
We name this method as ``MISC-p'', where ``-p'' stands for prioritization.

\textbf{Skill Discovery with MISC and DIAYN:}
One of the relevant works on unsupervised RL, DIAYN~\cite{eysenbach2018diversity}, introduces an information-theoretical objective $\mathcal{F}_{\text{DIAYN}}$, which learns diverse discriminable skills indexed by the latent variable $Z$, mathematically,
\begin{align}
\mathcal{F}_{\text{DIAYN}} \nonumber
          = I(S; Z) + \ent(A \mid S, Z)  \nonumber
          \ge  \EE_{\PM{Z}\PM{S}}[\log q_{\phi}(z \mid s) - \log p(z)] + \ent(A \mid S, Z). \nonumber
\end{align}
The first term, $I(S; Z)$, in the objective, $\mathcal{F}_{\text{DIAYN}}$, is implemented via a skill discriminator, which serves as a variational lower bound of the original objective~\cite{barber2003algorithm,eysenbach2018diversity}. 
The skill discriminator assigns high rewards to the agent, if it can predict the skill-options, $Z$, given the states, $S$. 
The second term, $\ent(A \mid S, Z)$, is implemented through SAC \cite{haarnoja2018soft} conditioned on skill-options~\cite{szepesvari2014universal}. 

We adapt DIAYN to goal-oriented tasks by replacing the full states, $S$, with goal states, $S^g$, as $I(S^g; Z)$.
In comparison, our method MISC proposes to maximize the MI between the controllable states and the goal states, $I(S^c; S^g)$. 
These two methods can be combined as follows:
\begin{align}
\mathcal{F}_{\text{MISC+DIAYN}} \nonumber
          &= I(S^c; S^g) + I(S^g; Z) + \ent(A \mid S, Z).
\end{align}
The combination of MISC and DIAYN helps the agent to learn control primitives via skill-conditioned policy for hierarchical RL \cite{eysenbach2018diversity}.

\textbf{Comparison and Combination with DISCERN:}
Another relevant work is Discriminative Embedding Reward Networks (DISCERN) \cite{warde2018unsupervised}, whose objective is to maximize the MI between the goal states $S^g$ and the environmental goals $G^e$, namely $I(S^g; G^e)$.
While MISC's objective is to maximize the MI between the controllable states $S^c$ and the goal states $S^g$, namely $I(S^c; S^g)$.
Intuitively, DISCERN attempts to reach a particular environment goal in each episode, while our method tries to manipulate the goal state to \textit{any} different value. 
MISC and DISCERN can be combined as 
\begin{align}
\nonumber
\mathcal{F}_{\text{MISC+DISCERN}} = I(S^c; S^g) + I(S^g; G^e).
\end{align} 
Through this combination, MISC helps DISCERN to learn its discriminative objective.


\section{Experiments}

To evaluate the proposed methods, we used the robotic manipulation tasks and also a navigation task, see Figure~\ref{fig:fetch3env1nav}~\cite{brockman2016openai,plappert2018multi}.
First, we analyze the control behaviors learned purely with the intrinsic reward (refer to the \href{https://youtu.be/CT4CKMWBYz0?t=4}{video starting from 0:04} and Figure 1 in Appendix~\ref{app:learned-control-behavior}). 
Secondly, we show that the pretrained models can be used for improving performance in conjunction with the task rewards. 
Interestingly, we show that the pretrained MI estimator can be transferred among different tasks and still improve performance. 
We compared MISC with other methods, including DDPG~\cite{lillicrap2015continuous}, SAC~\cite{haarnoja2018soft}, DIAYN~\cite{eysenbach2018diversity}, DISCERN~\cite{warde2018unsupervised}, PER~\cite{schaul2015prioritized}, VIME~\cite{houthooft2016vime}, ICM~\cite{pathak2017curiosity}, and Empowerment~\cite{mohamed2015variational}. 
Thirdly, we show some insights about how the MISC rewards are distributed across a trajectory. 
The experimental details are shown in Appendix~\ref{app:experimental-details}.
Our code is available at \url{https://github.com/ruizhaogit/misc} and \url{https://github.com/HorizonRobotics/alf}.


\begin{question}
What behavior does MISC learn?
\end{question}
We tested MISC in the robotic manipulation tasks.
The object is randomly placed on the table at the beginning of each episode.
During training, the agent only receives the intrinsic MISC reward.
In all three environments, the behavior of reaching objects emerges. 
In the push environments, the agent learns to push the object around on the table. 
In the slide environment, the agent learns to slide the object into different directions. 
Perhaps surprisingly, in the pick \& place environment, the agent learns to pick up the object from the table without any task reward.
All the observations are shown in the uploaded \href{https://youtu.be/CT4CKMWBYz0?t=4}{video starting from 0:04}.



\begin{question}
Can we use learned behaviors to directly maximize the task reward?
\end{question}

\begin{wrapfigure}[10]{r}{0.55\textwidth}
    \vspace{-1em}
    \includegraphics[width=\linewidth]{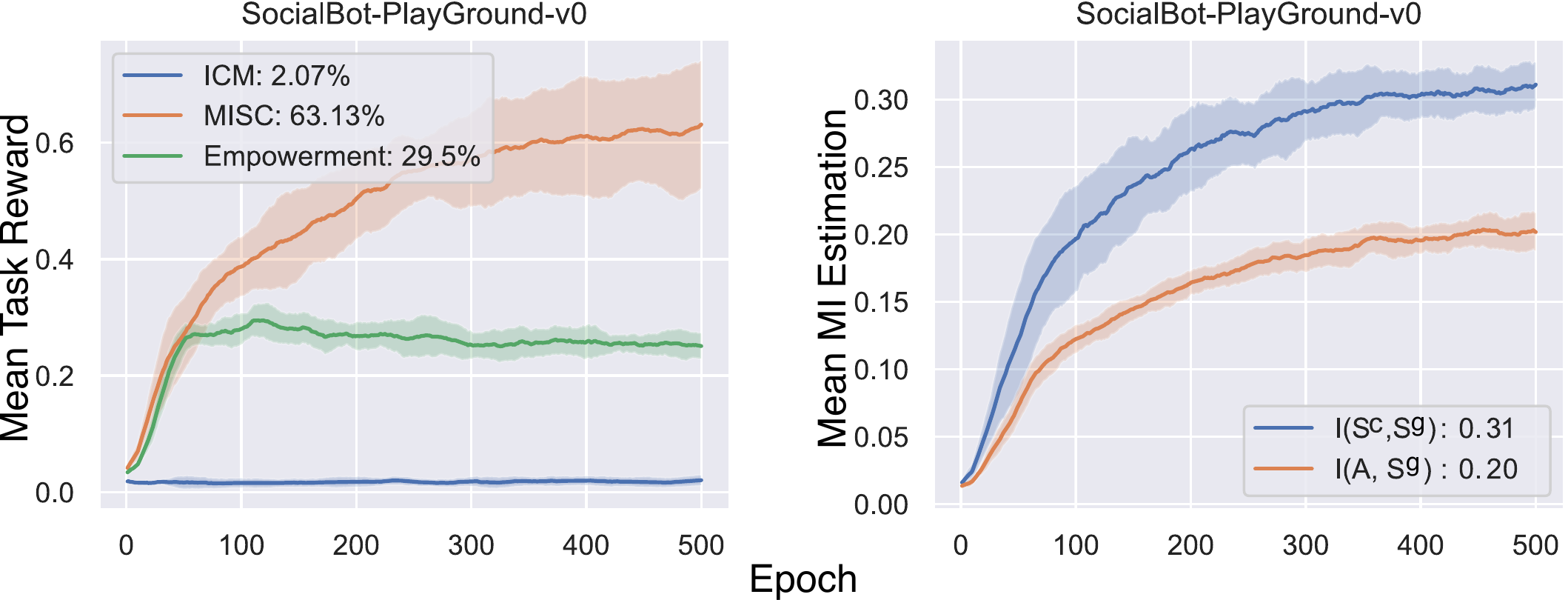}
    \vspace{-2em}
    \caption{\textbf{Experimental results in the navigation task}\label{fig:navigation-task}}
\end{wrapfigure}

We tested our method in the navigation task, which is based on the Gazebo simulator. 
The task reward is 1 if the agent reaches the ball, otherwise, the task reward is 0. 
We combined our method with PPO~\cite{schulman2017proximal} and compared the performance with ICM~\cite{pathak2017curiosity} and Empowerment~\cite{mohamed2015variational}. 
During training, we only used one of the intrinsic rewards such as MISC, ICM, or Empowerment to train the agent. 
Then, we used the averaged task reward as the evaluation metric. 
The experimental results are shown in Figure~\ref{fig:navigation-task} (left). The y-axis represents the mean task reward and the x-axis denotes the training epochs. 
From Figure~\ref{fig:navigation-task} (left), we can see that the proposed method, MISC, has the best performance. 
Empowerment has the second-best performance. 
Figure~\ref{fig:navigation-task} (right) shows that the MISC reward signal $I(S^c, S^g)$ is relatively strong compared to the Empowerment reward signal $I(A, S^g)$.
Subsequently, higher MI reward encourages the agent to explore more states with higher MI. 
A theoretical connection between Empowerment and MISC is shown in Appendix~\ref{app:proof-empowerment}. 
The \href{https://youtu.be/CT4CKMWBYz0?t=104}{video starting from 1:44} shows the learned navigation behaviors.


\begin{question}
How does MISC compare to DIAYN?
\end{question}
\begin{figure*}
  \centering
  \includegraphics[width=0.9\linewidth]{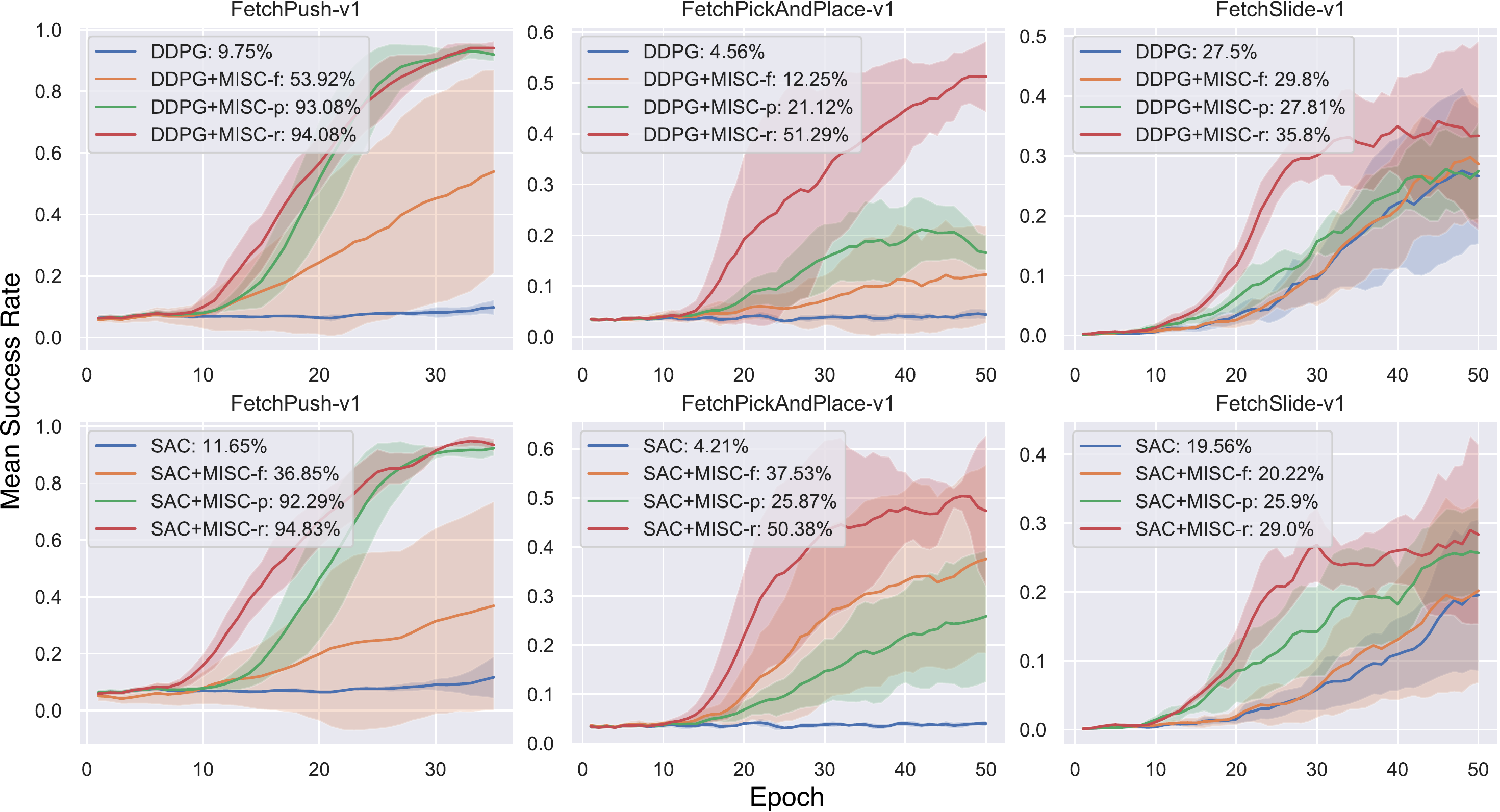} 
  \caption{\textbf{Mean success rate with standard deviation:} The percentage values after colon (:) represent the best mean success rate during training. The shaded area describes the standard deviation.}
  \label{fig:fig_accuracy}
\end{figure*}
We compared MISC, DIAYN and MISC+DIAYN in the pick \& place environment.
For implementing MISC+DIAYN, we first pre-train the agent with only MISC, and then fine-tune the policy with DIAYN. 
After pre-training, the MISC-trained agent learns manipulation behaviors such as, reaching, pushing, sliding, and picking up an object.
Compared to MISC, the DIAYN-trained agent rarely learns to pick up the object. It mostly pushes or flicks the object with the gripper. 
However, the combined model, MISC+DIAYN, learns to pick up the object and moves it to different locations, depending on the skill-option.
These observations are shown in the \href{https://youtu.be/CT4CKMWBYz0?t=48}{video starting from 0:48}.
In short, MISC helps the agent to learn the DIAYN objective. 
The agent first learns to control the object with MISC, and then discovers diverse manipulation skills with DIAYN.


\begin{question}
How does MISC+DISCERN compare to DISCERN?
\end{question}

\begin{wraptable}[8]{r}{0.65\textwidth}
\vspace{-1.7em}
\centering
\caption{\textbf{Comparison of DISCERN with and without MISC}}
\begin{tabular}{ p{2.7cm}    p{2.5cm}   p{2.6cm} } \toprule
Method  & Push (\%) & Pick \& Place (\%) \\ \midrule
DISCERN & 7.94\% $\pm$ 0.71\%  & 4.23\% $\pm$ 0.47\%  \\
R (Task Reward)  & 11.65\% $\pm$ 1.36\% & 4.21\% $\pm$ 0.46\%  \\
R+DISCERN & 21.15\% $\pm$ 5.49\% & 4.28\% $\pm$ 0.52\%  \\
R+DISCERN+MISC & 95.15\% $\pm$ 8.13\% & 48.91\% $\pm$ 12.67\% \\ \bottomrule
\end{tabular}
\label{tab:misc-discern}
\vspace{-1em}
\end{wraptable}

The combination of MISC and DISCERN, encourages the agent to learn to control the object via MISC and then move the object to the target position via DISCERN.
Table~\ref{tab:misc-discern} shows that DISCERN+MISC significantly outperforms DISCERN. 
This is because that MISC emphases more on state-control and teaches the agent to interact with an object.
Afterwards, DISCERN teaches the agent to move the object to the goal position in each episode.

\begin{question}
How can we use the learned behaviors or the MI estimator to accelerate learning?
\end{question}
We investigated three ways of using MISC to accelerate learning in addition to the task reward.
We combined these three variants with DDPG and SAC and tested them in the multi-goal robotic tasks. 
The environments, including push, pick \& place, and slide, have a set of predefined goals, which are represented as the red dots, as shown in Figure~\ref{fig:fetch3env1nav}. 
The task for the RL agent is to manipulate the object to the goal positions. We ran all the methods in each environment with 5 different random seeds and report the mean success rate and the standard deviation, as shown in Figure~\ref{fig:fig_accuracy}. 
The percentage values alongside the plots are the best mean success rates during training. 
Each experiment is carried out with 16 CPU-cores.
From Figure~\ref{fig:fig_accuracy}, we can see that all these three methods, including MISC-f, MISC-p, and MISC-r, accelerate learning in the presence of task rewards. 
Among these variants, the MISC-r has the best overall improvements. 
In the push and pick \& place tasks, MISC enables the agent to learn in a short period of time. 
In the slide tasks, MISC-r also improves the performances by a decent margin.
\begin{figure*}
  \centering
  \includegraphics[width=0.9\linewidth]{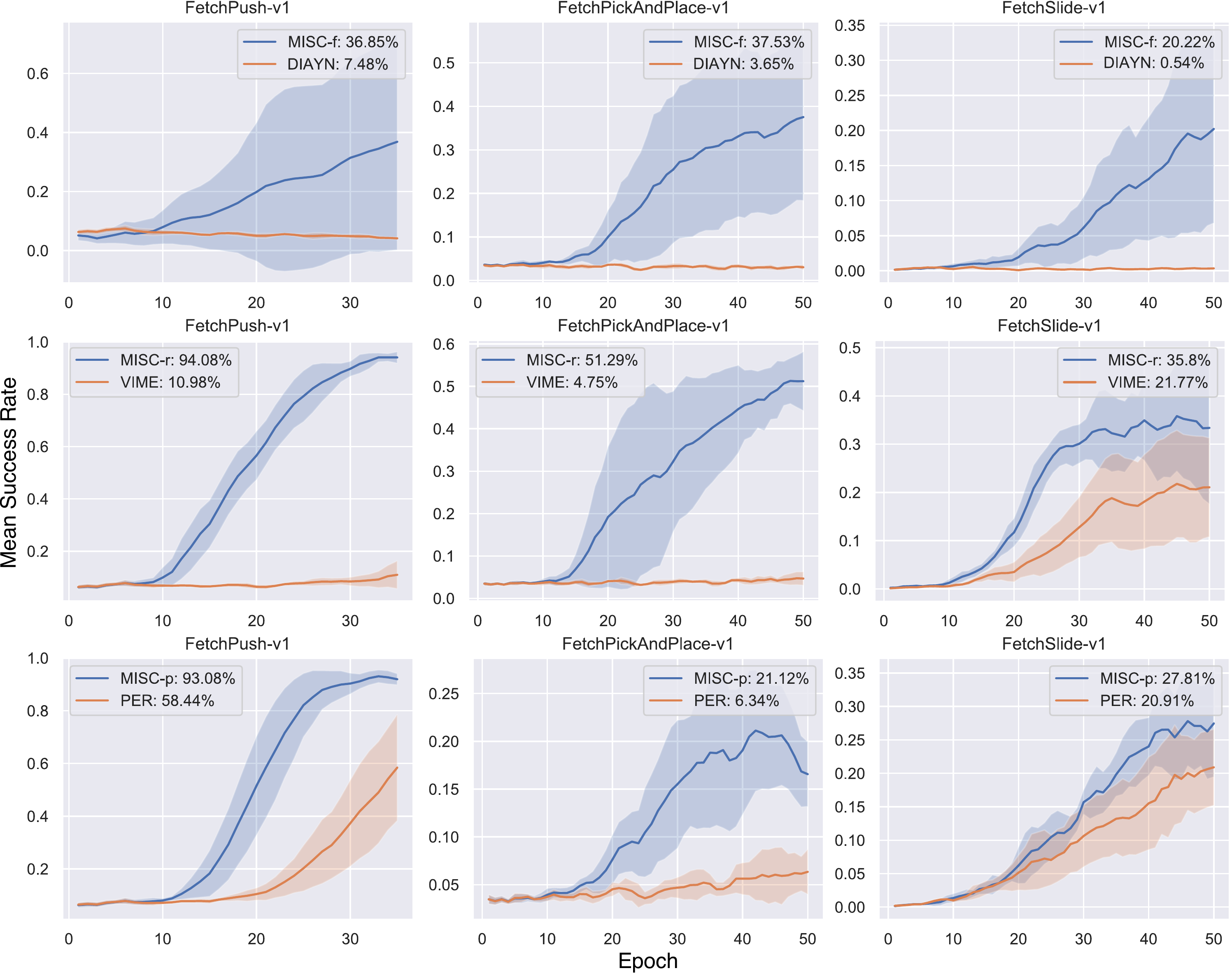} 
  \caption{\textbf{Performance comparison:} We compare the MISC variants, including MISC-f, MISC-r, and MISC-p, with DIAYN, VIME, and PER, respectively.}
  \label{fig:fig_compare}
\end{figure*}
We also compare our methods with more advanced RL methods. 
To be more specific, we compare MISC-f against the parameter initialization using DIAYN~\cite{eysenbach2018diversity}; MISC-p against Prioritized Experience Replay (PER), which uses TD-errors for prioritization~\cite{schaul2015prioritized}; and MISC-r versus Variational Information Maximizing Exploration (VIME)~\cite{houthooft2016vime}. 
The experimental results are shown in Figure~\ref{fig:fig_compare}. 
From Figure~\ref{fig:fig_compare} (1\textsuperscript{st} row), we can see that MISC-f enables the agent to learn, while DIAYN does not. 
In the 2\textsuperscript{nd} row of Figure~\ref{fig:fig_compare}, MISC-r performs better than VIME. 
This result indicates that the MI between states is a crucial quantity for accelerating learning. 
The MI intrinsic rewards boost performance significantly compared to VIME. 
This observation is consistent with the experimental results of MISC-p and PER, as shown in Figure~\ref{fig:fig_compare} (3\textsuperscript{rd} row), where the MI-based prioritization framework performs better than the TD-error-based approach, PER. 
On all tasks, MISC enables the agent to learn the benchmark task more quickly.


\begin{question}
Can the learned MI estimator be transferred to new tasks?
\end{question}
\begin{wrapfigure}[10]{r}{0.55\textwidth}
    \vspace{-1em}
    \includegraphics[width=\linewidth]{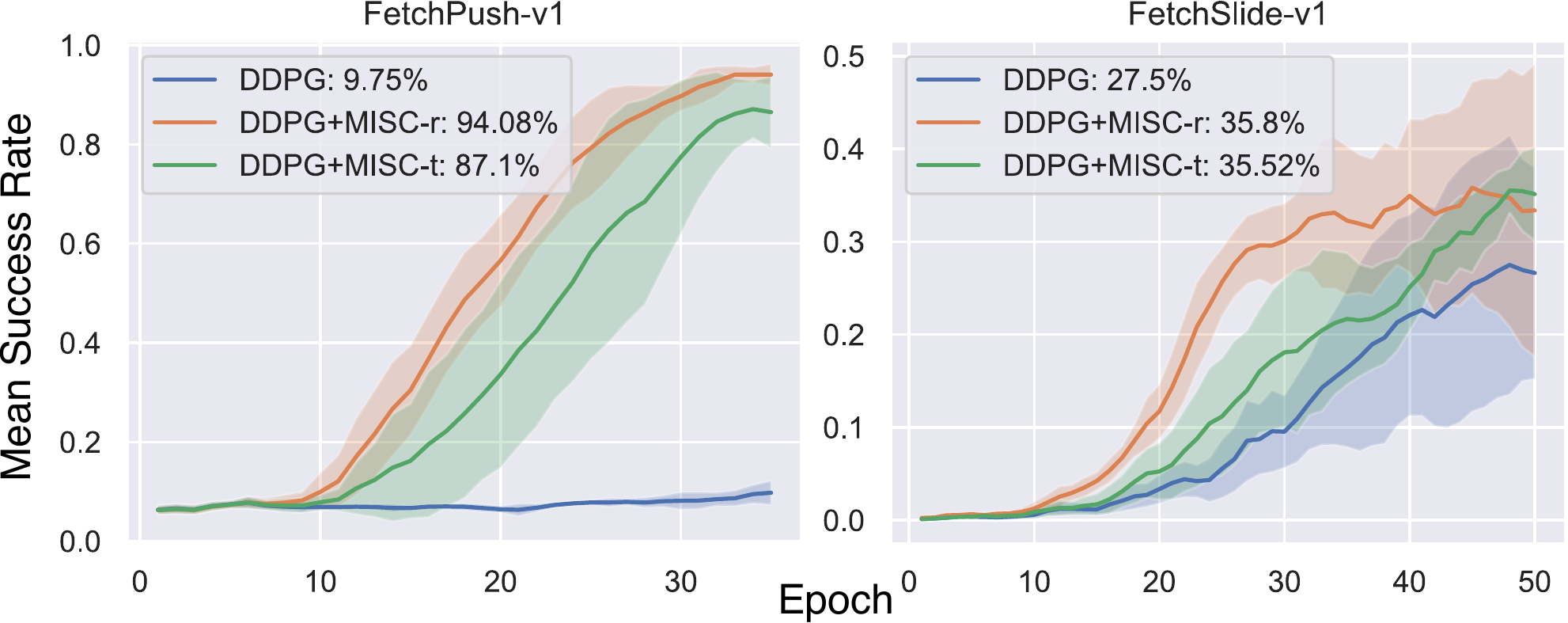}
    \vspace{-2em}
    \caption{\textbf{Transferred MISC}\label{fig:transfer}}
\end{wrapfigure}
It would be beneficial if the pretrained MI estimator could be transferred to a new task and still improve the performance~\cite{pan2010survey,bengio2012deep}. 
To verify this idea, we directly applied the pretrained MI estimator from the pick \& place environment to the push and slide environments, respectively. 
We denote this transferred method as ``MISC-t'', where ``-t'' stands for transfer. 
The MISC reward function trained in its corresponding environments is denoted as ``MISC-r''. 
We compared the performances of DDPG, MISC-r, and MISC-t. 
The results are in Figure~\ref{fig:transfer}, which shows that the transferred MISC still improved the performance significantly. 
Furthermore, as expected, MISC-r performed better than MISC-t. 
We can see that the MI estimator can be trained in a task-agnostic~\cite{finn2017model} fashion and later utilized in unseen tasks.


\begin{question}
How does MISC distribute rewards over a trajectory?
\end{question}
\begin{wrapfigure}[8]{r}{0.55\textwidth}
    \vspace{-1em}
    \includegraphics[width=\linewidth]{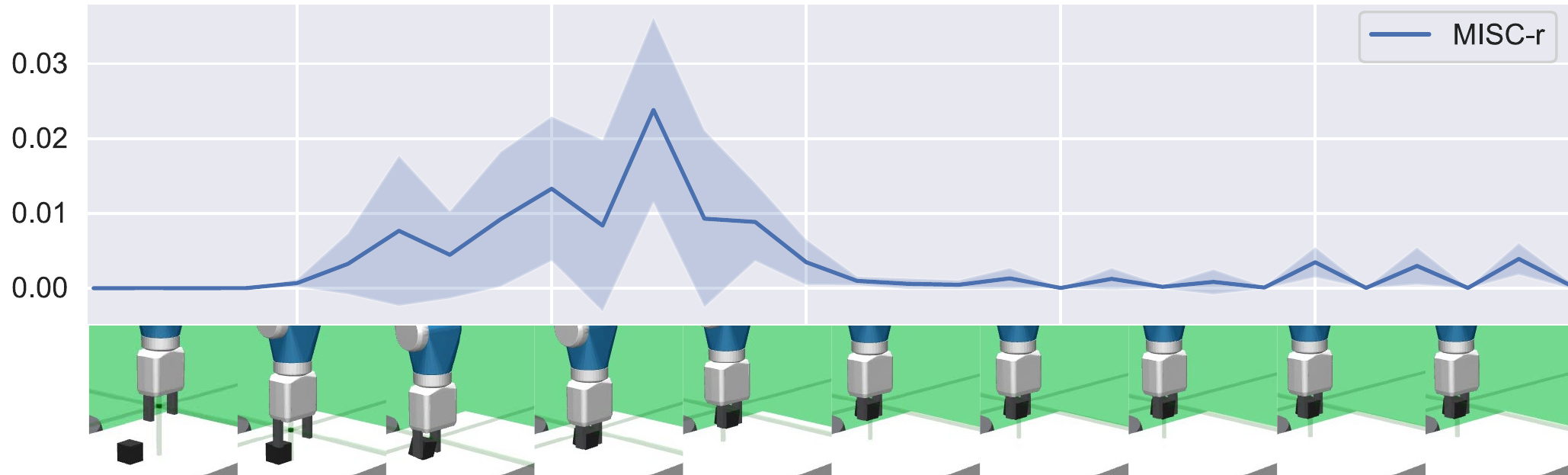}
    \vspace{-2em}
    \caption{\textbf{MISC rewards over a trajectory}\label{fig:misc-r}}
\end{wrapfigure}
To understand why MISC works, we visualize the learned MISC-r in Figure~\ref{fig:misc-r} and in the \href{https://youtu.be/CT4CKMWBYz0?t=92}{video starting from 1:32}. 
We can observe that the MI reward peaks between the 4th and 5th frame, where the robot quickly picks up the cube from the table. 
Around the peak reward value, the middle range reward values are corresponding to the relatively slow movement of the object and the gripper (see the 3rd, 9th, and 10th frame). 
When there is no contact between the gripper and the cube (see the 1st \& 2nd frames), or the gripper holds the object still (see the 6th to 8th frames) the intrinsic reward remains nearly zero. 
From this example, we see that MISC distributes positive intrinsic rewards when the goal state has correlated changes with the controllable state.


\begin{question}
Can MISC help the agent to learn when there are no objects or multiple objects? 
\end{question}
In the navigation task, we define the MISC objective to be the MI between the left wheel and the right wheel. 
We observe that the agent learns to balance itself and run in a straight line, 
as shown in the \href{https://youtu.be/CT4CKMWBYz0?t=134}{video starting from 2:14}. When there are multiple objects to control, we define the MISC objective as: 
$
\mathcal{F}_{\text{MISC}}= \sum_i I(S^c; S_i^g).
$
In the case that there is a red and a blue ball on the ground, with MISC, the agent learns to reach both balls and sometimes also learns to use one ball to hit the other ball. 
The results are shown in the uploaded \href{https://youtu.be/CT4CKMWBYz0?t=149}{video starting from 2:29}.



\textbf{Summary and Future Work:}
We can see that, with different combinations of the goal states and the controllable states, the agent is able to learn different control behaviors. 
When there are no specific goal states involved, we can train a skill-conditioned policy corresponding to different combinations of the two sets of states and later use the pretrained policy for the tasks at hand, see Appendix~\ref{app:automatic-goal-state-discovery} ``Automatic Discovery of Controllable States and Potential Goal States'' and Appendix~\ref{app:skill-discovery-hrl} ``Skill Discovery for Hierarchical Reinforcement Learning''.


\section{Related Work}

Deep RL led to great successes in various tasks \cite{ng2006autonomous,peters2008reinforcement,mnih2015human,levine2016end,zhao2018improving,zhao2018efficient,zhao2018learning}.
However, RL via intrinsic motivation is still a challenging topic.
Intrinsic rewards are often used to help the agent learn more efficiently to solve tasks. For example, \citet{jung2011empowerment} and \citet{mohamed2015variational} use empowerment, which is the channel capacity between states and actions, for intrinsically motivated RL agents. A theoretical connection between MISC and empowerment is shown in Appendix~\ref{app:proof-empowerment}.
VIME~\cite{houthooft2016vime} and ICM~\cite{pathak2017curiosity} use curiosity as intrinsic rewards to encourage the agents to explore the environment more thoroughly.
Another line of work on intrinsic motivation for RL is to discover meaningful skills.
Variational Intrinsic Control (VIC) \cite{gregor2016variational} proposes an information-theoretical objective~\cite{barber2003algorithm} to jointly maximize the entropy of a set of options while keeping the options distinguishable based on the final states of the trajectory. 
Recently, \citet{eysenbach2018diversity} introduced DIAYN, which maximizes the MI between a fixed number of skill-options and the entire states of the trajectory. \citet{eysenbach2018diversity} show that DIAYN can scale to more complex tasks compared to VIC and provides a handful of low-level primitive skills as the basis for hierarchical RL.
Intrinsic motivation also helps the agent to learn goal-conditioned policies. 
\citet{warde2018unsupervised} proposed DISCERN, a method to learn a MI objective between the states and goals, which enables the agent to learn to achieve goals in environments with continuous high-dimensional observation spaces. 
Based on DISCERN, \citet{pong2019skew} introduced Skew-fit, which adapts a maximum entropy strategy to sample goals from the replay buffer~\cite{zhao2019curiosity,zhao2019maximum} in order to make the agent learn more efficiently in the absence of rewards.
More recently, \citet{hartikainen2019dynamical} proposed to automatically learn dynamical distances, which are defined as a measure of the expected number of time steps to reach a given goal that can be used as intrinsic rewards for accelerating learning to achieve goals.
Based on a similar motivation as previous works, we introduce MISC, a method that uses the MI between the goal states and the controllable states as intrinsic rewards. 
In contrast to previous works on intrinsic rewards~\cite{mohamed2015variational,houthooft2016vime,pathak2017curiosity,eysenbach2018diversity,warde2018unsupervised}, MISC encourages the agent to interact with the interested part of the environment, which is represented by the goal state, and learn to control it. 
The MISC intrinsic reward is critical when controlling a specific subset of the environmental state is the key to complete the task, such as the case in robotic manipulation tasks.
Our method is complementary to these previous works, such as DIAYN and DISCERN, and can be combined with them. 
Inspired by previous works~\cite{schaul2015prioritized,houthooft2016vime,zhao2018energy,eysenbach2018diversity}, we additionally demonstrate three variants, including MISC-based fine-tuning, rewarding, and prioritizing mechanisms, to accelerate learning when the task rewards are available.


\section{Conclusion}

This paper introduces Mutual Information-based State-Control (MISC), an unsupervised RL framework for learning useful control behaviors.  
The derived efficient MI-based theoretical objective encourages the agent to control states without any task reward. 
MISC enables the agent to self-learn different control behaviors, which are non-trivial, intuitively meaningful, and useful for learning and planning.
Additionally, the pretrained policy and the MI estimator significantly accelerate learning in the presence of task rewards. 
We evaluated three MISC-based variants in different environments and demonstrate a substantial improvement in learning efficiency compared to state-of-the-art methods.


\section*{Broader Impact}

The broader impact of this work would be improving the learning efficiency of robots for continuous control tasks, such as navigation and manipulation.
In the future, when the learning robots are intelligent enough, they might be able to complete some of the repeatable or dangerous works for us.
Therefore, we could imagine the positive outcomes of this work for the society.


\bibliography{reference}
\bibliographystyle{plainnat}


\appendix

\section{Surrogate Objective}
\label{app:proof-surrogate}
\begin{lemma}
The mutual information quantity $I_{\phi}(S^g; S^c \mid \Tau)$ increases when we maximize the surrogate objective $\EE_{\PM{\Tau'}} [ I_{\phi}(S^g; S^c \mid \Tau')]$, mathematically,  
\begin{align}
\label{eq:mi-eq-surrogate}
I_{\phi}(S^g; S^c \mid \Tau) \ltimes \EE_{\PM{\Tau'}} [ I_{\phi}(S^g; S^c \mid \Tau')],
\end{align}
where $S^g$, $S^c$, and $\Tau$ denote goal states, controllable states, and trajectories, respectively. The trajectory fractions are defined as the adjacent state pairs, namely $\Tau'=\{S_{t}, S_{t+1}\}$.
The symbol $\ltimes$ denotes a monotonically increasing relationship between two variables and $\phi$ represents the parameter of the statistics model in MINE.
\end{lemma}

\begin{proof}
The derivation of the MI surrogate objective in Equation~(\ref{eq:mi-eq-surrogate}) is shown as follows:
\begin{align}
\label{eq:mi-eq-donsker}
I_{\phi}(S^g; S^c \mid \Tau)
=& \EE_{\PJ{S^g}{S^c \mid \Tau}}[\FGen_\phi] - \log(\EE_{\PI{S^g \mid \Tau}{S^c \mid \Tau}}[e^{\FGen_\phi}]) \\
\label{eq:mi-eq-x}
\ltimes& \EE_{\PJ{S^g}{S^c \mid \Tau}}[\FGen_\phi] - \EE_{\PI{S^g \mid \Tau}{{S}^c \mid \Tau}}[e^{\FGen_\phi}]\\
\label{eq:mi-eq-fraction-x}
=& \EE_{\PM{\Tau'}} [ \EE_{\PJ{S^g}{S^c \mid \Tau'}}[\FGen_\phi] - \EE_{\PI{S^g \mid \Tau'}{S^c \mid \Tau'}}[e^{\FGen_\phi}] ] \\
\label{eq:mi-eq-fraction-log}
\ltimes& \EE_{\PM{\Tau'}} [ \EE_{\PJ{S^g}{S^c \mid \Tau'}}[\FGen_\phi] - \log(\EE_{\PI{S^g \mid \Tau'}{S^c \mid \Tau'}}[e^{\FGen_\phi}]) ]
= \EE_{\PM{\Tau'}} [ I_{\phi}(S^g; S^c \mid \Tau')],
\end{align}
where $T_{\phi}$ represents a neural network, whose inputs are state samples and the output is a scalar.
For simplicity, we use the symbol $\ltimes$ to denote a monotonically increasing relationship between two variables, for example, $\log(x) \ltimes x$ means that as the value of $x$  increases, the value of $\log(x)$ also increases and vice versa.
To decompose the lower bound Equation~(\ref{eq:mi-eq-donsker}) into small parts, we make the following derivations, see Equation~(\ref{eq:mi-eq-x},\ref{eq:mi-eq-fraction-x},\ref{eq:mi-eq-fraction-log}).
Deriving from Equation~(\ref{eq:mi-eq-donsker}) to Equation~(\ref{eq:mi-eq-x}), we use the property that $\log(x) \ltimes x$.
Here, the new form, Equation~(\ref{eq:mi-eq-x}), allows us to decompose the MI estimation into the expectation over MI estimations of each trajectory fractions, Equation~(\ref{eq:mi-eq-fraction-x}). 
To be more specific, we move the implicit expectation over trajectory fractions in Equation~(\ref{eq:mi-eq-x}) to the front, and then have Equation~(\ref{eq:mi-eq-fraction-x}).
The quantity inside the expectation over trajectory fractions is the MI estimation using only each trajectory fraction, see Equation~(\ref{eq:mi-eq-fraction-x}). We use the property, $\log(x) \ltimes x$, again to derive from Equation~(\ref{eq:mi-eq-fraction-x}) to Equation~(\ref{eq:mi-eq-fraction-log}).
\end{proof}

\section{Connection to Empowerment}
\label{app:proof-empowerment}
The state $S$ contains the goal state $S^g$ and the controllable state $S^c$. 
For example, in robotic tasks, the goal state and the controllable state represent the object state and the end-effector state, respectively. 
The action space is the change of the gripper position and the status of the gripper, such as open or closed. 
Note that, the agent's action directly alters the controllable state.

Here, given the assumption that the transform, $S^c=F(A)$, from the action, $A$, to the controllable state, $S^c$, is a smooth and uniquely invertible mapping~\cite{kraskov2004estimating}, then we can prove that the MISC objective, $I(S^c, S^g)$, is equivalent to the empowerment objective, $I(A, S^g)$.

The empowerment objective~\cite{klyubin2005empowerment,salge2014empowerment,mohamed2015variational} is defined as the channel capacity in information theory, which means the amount of information contained in the action $A$ about the state $S$, mathematically:
\begin{align}
\label{eq:empowerment}
\mathcal{E} = I(S, A).
\end{align}
Here, we replace the state variable $S$ with goal sate $S^g$, we have the empowerment objective as follows,
\begin{align}
\label{eq:empowerment_i}
\mathcal{E} = I(S^g, A).
\end{align}

\begin{theorem}
The MISC objective, $I(S^c, S^g)$, is equivalent to the empowerment objective, $I(A, S^g)$, given the assumption that the transform, $S^c=F(A)$, is a smooth and uniquely invertible mapping:
\label{th:misc=empowerment}
\begin{align}
\label{eq:misc=empowerment}
I(S^c, S^g) = I(A, S^g)
\end{align}
\end{theorem}
where $S^g$, $S^c$, and $A$ denote the goal state, the controllable state, and the action, respectively.
\begin{proof}
\begin{align}
\label{eq:equivalence-prove}
I(S^c, S^g) &= \int \int d{s^c} d{s^g} p(s^c,s^g) \log \frac{p(s^c,s^g)}{p(s^c)p(s^g)}  \\ 
            &= \int \int d{s^c} d{s^g} \left\|\frac{\partial A}{\partial S^c}\right\| p(a,s^g) \log \frac{\left\|\frac{\partial A}{\partial S^c}\right\| p(a,s^g)}{\left\|\frac{\partial A}{\partial S^c}\right\|p(a)p(s^g)}  \\
            &= \int \int d{s^c} d{s^g} J_{A}(s^c) p(a,s^g) \log \frac{J_{A}(s^c) p(a,s^g)}{J_{A}(s^c)p(a)p(s^g)}  \\ 
            &= \int \int d{a} d{s^g} p(a,s^g) \log \frac{p(a,s^g)}{p(a)p(s^g)}  \\
            &= I(A, S^g)
\end{align}
\end{proof}


\section{Learned Control Behaviors without Supervision}
\label{app:learned-control-behavior}
The learned control behaviors without supervision are shown in Figure~\ref{fig:all4skills} as well as in the uploaded \href{https://youtu.be/CT4CKMWBYz0?t=4}{video starting from 0:04}.
\begin{figure*}[h]
  \centering
  \includegraphics[width=\linewidth]{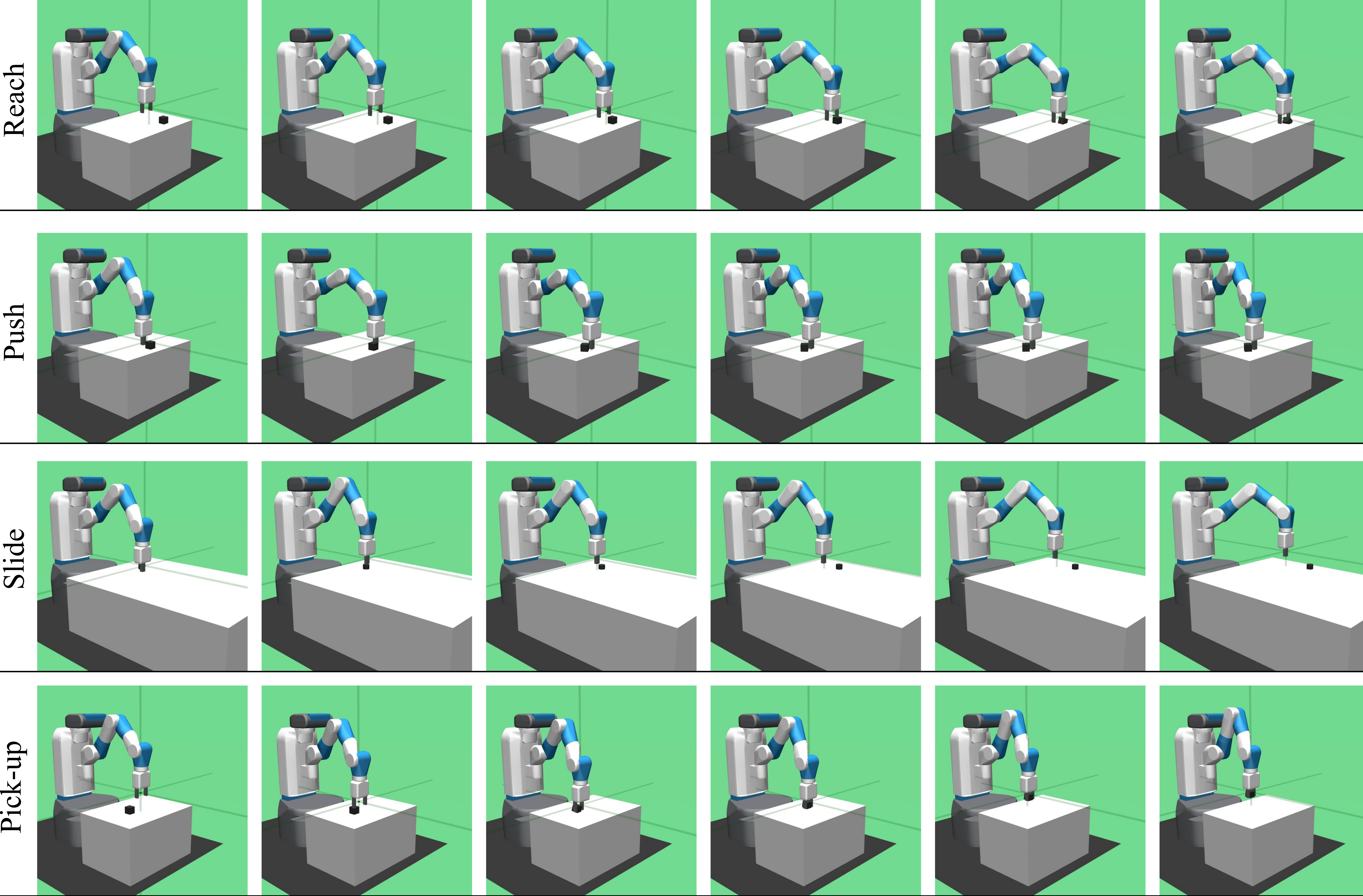}
  \caption{\textbf{Learned Control behaviors with MISC}: Without any reward, MISC enables the agent to learn control behaviors, such as reaching, pushing, sliding, and picking up an object. The learned behaviors are shown in the uploaded \href{https://youtu.be/CT4CKMWBYz0?t=4}{video starting from 0:04}.}
  \label{fig:all4skills}
\end{figure*}


\section{Comparison of Variational MI-based and MINE-based Implementations}
\label{app:compare-two-mi}

Here, we compare the variational approach-based \cite{barber2003algorithm} implementation of MISC and MINE-based implementation \cite{belghazi2018mine} of MISC in Table~\ref{tab:im-mine}.
All the experiments are conducted with 5 different random seeds. 
The performance metric is mean success rate (\%) $\pm$ standard deviation.
The “Task-r” stands for the task reward. 
\begin{table}[h]
\centering
\caption{\textbf{Comparison of variational MI (v)-based and MINE (m)-based MISC}}
\begin{tabular}{ p{4.cm}    p{3.cm}   p{3.cm} } \toprule
Method  & Push (\%) & Pick \& Place (\%) \\ \midrule
Task-r+MISC(v) & 94.9\% $\pm$ 5.83\%  & 49.17\% $\pm$ 4.9\%  \\
Task-r+MISC(m)  & 94.83\% $\pm$ 4.95\% & 50.38\% $\pm$ 8.8\%  \\ \bottomrule
\end{tabular}
\label{tab:im-mine}
\end{table}
From Table~\ref{tab:im-mine}, we can see that the performance of these two MI estimation methods are similar. 
However, the variational method introduces additional complicated sampling mechanisms, and two additional hyper-parameters, i.e., the number of the candidates and the type of the similarity measurement \cite{barber2003algorithm,eysenbach2018diversity,warde2018unsupervised}. 
In contrast, MINE-style MISC is easier to implement and has less hyper-parameters to tune. 
Furthermore, the derived surrogate objective improves the scalability of the MINE-style MISC.


\section{Automatic Discovery of Controllable States and Potential Goal States}
\label{app:automatic-goal-state-discovery}

When there are no specific goal states involved, we can train a skill-conditioned policy corresponding to different combinations of the two sets of states and later use the pretrained policy for the tasks at hand.
The controllable states can be automatically determined based on the MI between action $A$ and a state variable $S^{i}$, $I(A, S^i)$.
If the MI value is relatively high, then the $i$-th state variable is determined as controllable. Other states with relatively low MI with actions are considered as potential goal states.

For example, in the Fetch robot arm pick \& place environment, we have the follow states as the observation: \texttt{grip\_pos}, \texttt{grip\_velp}, \texttt{object\_pos}, \texttt{object\_velp}, \texttt{object\_rot}, \texttt{object\_velr}, where the abbreviation ``\texttt{pos}'' stands for position; ``\texttt{rot}'' stands for rotation; ``\texttt{velp}'' stands for linear velocity, and ``\texttt{velr}'' stands for rotational velocity.
In Table~\ref{tab:mi-action-state}, we show the MI estimation between action and each state based on a batch of random rollout trajectories.
\begin{table}[h]
\centering
\caption{\textbf{Mutual information estimation between action and state}}
\begin{tabular}{ p{5.cm}    p{3.cm}  } \toprule
Mutual Information  & Value  \\ \midrule
MI(\texttt{action}; \texttt{grip\_pos}) & 0.202 $\pm$ 0.142 \\
MI(\texttt{action}; \texttt{grip\_velp}) & 0.048 $\pm$ 0.043 \\
MI(\texttt{action}; \texttt{object\_pos}) & 0.000 $\pm$ 0.001 \\
MI(\texttt{action}; \texttt{object\_velp}) & 0.034 $\pm$ 0.030 \\
MI(\texttt{action}; \texttt{object\_rot}) & 0.018 $\pm$ 0.054 \\
MI(\texttt{action}; \texttt{object\_velr}) & 0.006 $\pm$ 0.018   \\ \bottomrule
\end{tabular}
\label{tab:mi-action-state}
\end{table}
From Table~\ref{tab:mi-action-state}, we can see that the state random variable \texttt{grip\_pos} has the highest MI with the action random variable. Therefore, \texttt{grip\_pos} is determined as controllable states.
In contrast, the state random variable \texttt{object\_pos} has the lowest MI with actions. Thus, we consider \texttt{object\_pos} as a potential goal states, which the agent should learn to control.

\section{Skill Discovery for Hierarchical Reinforcement Learning}
\label{app:skill-discovery-hrl}

In this section, we consider \texttt{grip\_pos} and \texttt{grip\_velp} as controllable states, and states, including \texttt{object\_pos}, \texttt{object\_velp}, \texttt{object\_rot}, \texttt{object\_velr} as potential goal states, based on Table~\ref{tab:mi-action-state}.
In Table~\ref{tab:mi-prior-post}, we show the MI value with different state-pair combinations prior to training and post to training. When the MI value difference is high, it means that the agent has a good learning progress with the corresponding MI objective.
\begin{table}[h]
\centering
\caption{\textbf{Mutual Information estimation prior and post to the training}}
\begin{tabular}{ p{5.cm}    p{4.cm}  p{4.cm}  } \toprule
Mutual Information Objective  & Prior-train Value & Post-train Value  \\ \midrule
MI(\texttt{grip\_pos}; \texttt{object\_pos}) & 0.003 $\pm$ 0.017 & 0.164 $\pm$ 0.055 \\
MI(\texttt{grip\_pos}; \texttt{object\_rot}) & 0.017 $\pm$ 0.084 & 0.461 $\pm$ 0.088 \\
MI(\texttt{grip\_pos}; \texttt{object\_velp}) & 0.005 $\pm$ 0.010 & 0.157 $\pm$ 0.050 \\
MI(\texttt{grip\_pos}; \texttt{object\_velr}) & 0.016 $\pm$ 0.083 & 0.438 $\pm$ 0.084 \\
MI(\texttt{grip\_velp}; \texttt{object\_pos}) & 0.004 $\pm$ 0.024 & 0.351 $\pm$ 0.213 \\
MI(\texttt{grip\_velp}; \texttt{object\_rot}) & 0.019 $\pm$ 0.092 & 0.420 $\pm$ 0.043 \\
MI(\texttt{grip\_velp}; \texttt{object\_velp}) & 0.005 $\pm$ 0.011 & 0.001 $\pm$ 0.002 \\
MI(\texttt{grip\_velp}; \texttt{object\_velr}) & 0.015 $\pm$ 0.081 & 0.102 $\pm$ 0.063   \\ \bottomrule
\end{tabular}
\label{tab:mi-prior-post}
\end{table}
From Table~\ref{tab:mi-prior-post} first row, we can see that with the intrinsic reward MI(\texttt{grip\_pos}; \texttt{object\_pos}), the agent achieves a high MI after training, which means that the agent learns to better control the object positions using its gripper. Similarly, in the second row of the table, with MI(\texttt{grip\_pos}; \texttt{object\_rot}), the agent learns to control object rotation with its gripper.
In contrast, from the second last row in the table, we can see that with MI(\texttt{grip\_velp}; \texttt{object\_velp}), the agent did not learn anything.

From the experimental results, we can see that with different combination of state-pairs of controllable and goal states, the agent can learn different skills, such as manipulate object positions or rotations.
We can connect these learned skills with different skill-options \cite{eysenbach2018diversity} and train a meta-controller to control these motion primitives to complete long-horizon tasks in a hierarchical reinforcement learning framework \cite{eysenbach2018diversity}. We consider this as a future research direction, which could be a solution in solving more challenging and complex long-horizon tasks.


\section{Experimental Details}
\label{app:experimental-details}

The experiments of the robotic manipulation tasks in this paper use the following hyper-parameters:
\begin{itemize}
    \item Actor and critic networks: $3$ layers with $256$ units each and ReLU non-linearities
    \item Adam optimizer~\cite{kingma2014adam} with $1\cdot10^{-3}$ for training both actor and critic
    \item Buffer size: $10^6$ transitions
    \item Polyak-averaging coefficient: $0.95$
    \item Action L2 norm coefficient: $1.0$
    \item Observation clipping: $[-200, 200]$
    \item Batch size: $256$
    \item Rollouts per MPI worker: $2$
    \item Number of MPI workers: $16$
    \item Cycles per epoch: $50$
    \item Batches per cycle: $40$
    \item Test rollouts per epoch: $10$
    \item Probability of random actions: $0.3$
    \item Scale of additive Gaussian noise: $0.2$
    \item Scale of the mutual information reward: $5000$
\end{itemize}

All hyper-parameters are described in greater detail at \url{https://github.com/ruizhaogit/misc/tree/master/params}.

\end{document}